\documentclass{mc}

\usepackage[
  paper=a4paper,         % or letterpaper
  left=25mm,
  right=25mm,
  top=25mm,
  bottom=25mm,
  includeheadfoot,       % count header/footer inside margins
  heightrounded%         % avoid underfull vboxes
]{geometry}

\usepackage{booktabs}
\usepackage{multirow}
\usepackage{makecell}

\usepackage[utf8]{inputenc}
\usepackage{textgreek}
\usepackage{tikz}
\renewcommand\datereceived{}
\renewcommand\dateaccepted{}
\usepackage{mathrsfs}
\usepackage{caption}
\usepackage{subcaption}
\usepackage{multirow}
\usepackage{xcolor,sectsty}
\usepackage{booktabs}

\usepackage{algorithm}
\usepackage{algpseudocode}

\usepackage{titlesec}
\titleformat{\subsection}
  {\normalsize\bfseries} 
  {\thesection.\arabic{subsection}} 
  {1em}
  {\textit}

\titleformat{\subsubsection}
  {\normalsize}
  {\thesubsection.\arabic{subsubsection}}
  {1em}
  {\textit}

\definecolor{astral}{RGB}{46,116,181}
\definecolor{darkslategray}{rgb}{0.18, 0.31, 0.31}
\definecolor{warmblack}{rgb}{0.0, 0.46, 0.36}
\sectionfont{\color{black}\normalsize}

\usepackage{hyperref}
\begin{document}
	
	%%%%% title and author(s):
	% \markboth{Author(s)}{Short Title}
	% \title{Title}
	
	\title{{LiNO: Lifting based multiresolution neural operator}}
	
	% single author:
	% \author[AUTHOR]{AUTHOR}
	% \address{address of AUTHOR}
	% \email{{\tt email address of AUTHOR} (AUTHOR)}
	
%	\author[D. Juki\'c]{Dragan Juki\'c\corrauth}
%	\address{Department of Mathematics, University of Osijek, Trg Ljudevita Gaja 6, HR-31 000 Osijek, Croatia}
%	\email{{\tt jukicd@mathos.hr} (D. Juki\'c)}

	% multiple authors:
	% Please mark \corrauth after the name of the corresponding author.
	% different addresses:
	% \author[AUTHOR1 and AUTHOR2]{AUTHOR1\affil{1}\comma\corrauth and AUTHOR2\affil{2}}
	% \address{\affilnum{1}\ address of AUTHOR1\\   \affilnum{2}\ address of AUTHOR2}

	%same address:

% \iffalse 
    
	\author[Pandey {\em et al.}]{
    Himanshu Pandey\affil{1},
    Subham Patel\affil{1,2},
    and Ratikanta Behera\affil{1}
    }
    
    \address{
    \affilnum{1}Department of Computational and Data Sciences,
    Indian Institute of Science, Bangalore 560012, India
    \affilnum{2} IISc Mathematics Initiative, Department of Mathematics, Indian Institute of Science, Bangalore, 560012, India
    }
    
    \email{
    {\tt phimanshu@iisc.ac.in},
    {\tt subhampatel@iisc.ac.in},
    {\tt ratikanta@iisc.ac.in}
    }

\begin{abstract}
Recently, neural operators have shown promising outcomes for learning solution operators of differential equations directly from data. This framework learns a functional mapping from the parameter field to the solution field, enabling the prediction of an entire class of solutions rather than a specific instance. However, existing operators often struggle to capture both global dynamics and fine-scale structure altogether. To design an effective operator capable of representing multiscale features, a hierarchical multiscale decomposition framework is required. In this study, we develop the lifting neural operator (LiNO), a multiresolution operator built on the second‑generation wavelet lifting scheme. LiNO learns a multiresolution decomposition directly from data by parameterizing the lifting transform. This lifting transformation is adaptive to the underlying solution function and exactly invertible by construction, enabling information-preserving multiscale operator learning. In the lifted multiresolution space, the operator evolves coarse and directional detail coefficients separately, resulting in scale-aware modeling of the underlying physics.  We evaluate LiNO on several benchmarks including Darcy flow, Poisson equation, Allen-Cahn equation, compressible Navier-Stokes equation and Gray-Scott reaction-diffusion system. Altogether they cover a wide range of physical behavior, from multiscale phenomena, transport dominated and chaotic dynamics. LiNO shows a strong performance on these challenging benchmarks compared to state-of-the-art neural operators. The results suggest that adaptive multiresolution operator is a promising direction for scientific machine learning.

\end{abstract}

%%%%%%%%%%%%%%%%%%%%%% end %%%%%%%%%%%
%%%%% Keywords %%%%%%%%%%%
	
%%%% AMS subject classifications %%%%
% \ams{26A33, 68T07 35R30, 65Mxx}
	
%%%% maketitle %%%%%
\maketitle	
%%%% Start %%%%%%%%%%%%%%%%%%%%%%%%%%%%%%%%%%%%%%%%%%%%%%%%%%%%%%%%%%%%%%%%%%%%%%%%%%%%%%%%%%%%%%%%%%%%%%%%%%%%%%%%%%

\section{Introduction}\label{sec: Intro}

% Many scientific and engineering systems involve partial differential equations (PDEs) that require a solution field given a set of input parameters. Example includeTo name a few, fluid dynamics, reaction--diffusion systems, climate prediction, and phase-field evolution are some of them. Such systems often exhibit strongly nonlinear, multiscale, and chaotic dynamics involving localized coherent structures, transport-dominated evolution, and long-range temporal dependencies. Traditional numerical solvers such as finite difference \cite{fdm}, finite element \cite{hughes2003finite}, and spectral methods \cite{boyd2001chebyshev, canuto2007spectral} often require repeated expensive simulations for varying parameters, forcing terms, or initial conditions, making large-scale uncertainty quantification, inverse design, and real-time forecasting computationally expensive.

Many scientific and engineering systems involve partial differential equations (PDEs) that require an estimate of the solution field for a given set of input parameters. Such applications arise in many areas, including fluid dynamics, reaction--diffusion systems, climate prediction, and phase-field evolution. These systems are typically characterized by strongly nonlinear, multiscale, and chaotic dynamics involving localized coherent structures, transport-dominated evolution, and long-range temporal dependencies. Traditional numerical solvers such as finite difference \cite{fdm}, finite element \cite{hughes2003finite}, and spectral methods \cite{boyd2001chebyshev, canuto2007spectral} often require repeated expensive simulations for varying parameters, forcing terms, or initial conditions, making large-scale uncertainty quantification, inverse design, and real-time forecasting computationally expensive.

Neural operators have recently emerged as a powerful framework for learning mappings between infinite-dimensional function spaces directly from data \cite{kovachki}. Early developments such as deep operator networks (DeepONet) \cite{DeepONet} demonstrated that neural networks can approximate nonlinear operators through branch--trunk architectures. DeepONet established a theoretical foundation for operator learning and showed strong approximation capability across various parametric PDE problems. Subsequently, the Fourier neural operator (FNO) \cite{FNO} and its variants 
% \cite{YOU2026114530, LEHMANN2025113813}
 \cite{LI2025117732, LEHMANN2024116718} 
demonstrated remarkable success by parameterizing integral kernels in Fourier space and efficiently learning global convolution operators through spectral convolutions. With the fast Fourier transform (FFT), FNO achieves strong computational complexity and remarkable performance on several benchmark PDE systems, particularly for smooth dynamics and periodic domains. However, FNO inherently relies on global spectral representations and low-frequency mode truncation. While effective for smooth dynamics, such spectral truncation may irreversibly discard fine-scale information that is physically important in multiscale PDE systems. Furthermore, the global nature of Fourier representations can limit locality adaptation and reduce representational efficiency for strongly nonlinear systems involving sharp interfaces, localized structures, and multiscale phenomena. 
%{\color{red} spectral bias \cite{zhi2020frequency}}

To address some of these limitations, several neural operator architectures have been proposed. Wavelet neural operators (WNO) \cite{WNO} introduced wavelet-domain operator learning to better capture localized spatial structures through multiresolution analysis \cite{mallat1989multifrequency, daubechies1992ten}. Since wavelets provide simultaneous spatial and frequency localization, WNO improves representation of multiscale features compared to purely global spectral methods. However, WNO relies on predefined wavelet bases and fixed transform operators. Consequently, the decomposition itself remains non-adaptive to the nonlinear dynamics of the underlying solution. This suggests that while multiresolution representations are important, the adaptability and learnability of the decomposition mechanism itself play a critical role in robust neural operator learning.

Beyond FNO and WNO, several other spectral neural operators are also explored, such as the Laplace neural operator (LNO) \cite{LNO}. The LNO performs operator learning in the Laplace domain via learnable pole--residue representation. This formulation is well-suited for systems that evolve rapidly, such as impulse responses and wave-damping phenomena. However, like other transform-based neural operators, it does not explicitly support adaptive multiresolution operator learning. Subsequent advances, including U-Net-based operators (UNO) \cite{UNO} and convolutional neural operators (CNO) \cite{CNO} explored hierarchical multiscale representations via localized convolutions and encoder--decoder architectures. These methods effectively learn local structure and provide better feature propagation across multi-resolution. Nevertheless, repeated pooling and interpolation operations tend to create information bottlenecks and irreversible compression artifacts.

Despite significant progress in neural operator learning, several limitations remain. Most existing operators struggle to learn both global dynamics and fine-scale structure simultaneously \cite{Azizzadenesheli2024}. Furthermore, transform-based operators rely on fixed bases, which may not adapt effectively to a problem-specific solution field. In contrast, convolution and pooling-based operators are generally not exactly invertible and often lose physically important structures during long-horizon evolution \cite{long2026stft}. These issues are particularly evident in strongly nonlinear and chaotic PDE systems, where stable and accurate autoregressive predictions remain challenging.

% Despite substantial progress in neural operator learning, several fundamental challenges remain. Existing architectures often fail to capture both global dynamics and fine-scale localized structures effectively \cite{ Azizzadenesheli2024}. Transform-based approaches rely on fixed bases, which may not adapt effectively to a problem-specific solution field. Furthermore, information loss introduced through spectral truncation or pooling-based hierarchies can progressively degrade physically important structures during long-horizon evolution \cite{long2026stft}. These limitations are particularly evident in strongly nonlinear and chaotic PDE systems, where stable and accurate autoregressive predictions are still a challenge.

To address these challenges, we propose the lifting neural operator (LiNO), a multiresolution neural operator inspired by the lifting scheme \cite{SWELDENS1996186, Sweldens1998}. The lifting scheme has been used to develop adaptive numerical methods for solving PDEs \cite{vasilyev2000second}, and extended to multilevel operator approximation on curved and complex-geometry domains \cite{BeheraMehra15, MehraKev08}. Unlike transform-based operators that rely on fixed analytical transforms, LiNO formulates the lifting procedure as a learnable neural operator framework in which the transform itself is parameterized using neural networks and optimized directly from data. This formulation enables the multiresolution basis to adapt the underlying solution field, improving representation of nonlinear localized structures and multiscale interactions. Another key advantage of the lifting transform is its exact invertibility by construction, guaranteeing lossless decomposition and reconstruction without requiring orthogonality constraints or spectral truncation. Consequently, LiNO avoids irreversible information loss commonly encountered in Fourier truncation and pooling-based architectures.

We evaluate the proposed framework on several benchmark PDEs, including the Allen--Cahn equation, Darcy flow, Poisson equation, compressible Navier--Stokes equations in both viscous and inviscid regimes, and the Gray--Scott reaction--diffusion system. Experimental results demonstrate that LiNO consistently achieves superior or highly competitive accuracy compared with state-of-the-art neural operators, including FNO, WNO, UNO, CNO, and LNO, across these challenging benchmarks. In particular, LiNO exhibits especially strong performance for problems where preserving localized fine-scale structures is critical, such as inviscid Navier--Stokes turbulence, Allen--Cahn interface dynamics, and Gray--Scott structure formation, highlighting the effectiveness of adaptive invertible multiresolution operator learning.

The main contributions of this work are summarized as follows:

\begin{itemize}
    \item We propose LiNO, a novel lifting-based neural operator framework that integrates the second-generation wavelet lifting scheme into neural operator learning, providing a fully invertible and learnable multiresolution decomposition tailored to PDE training data.

    \item The proposed framework deploys learnable transform via predict and update operators that adapt the multiresolution decomposition directly to the intrinsic structure of PDE solution manifolds.

    \item LiNO offers a scale-aware operator evolution mechanism in the lifted multiresolution space that separately propagates coarse and directional detail coefficients while preserving localized physical interactions.

    \item We validate the effectiveness of LiNO through extensive experiments, achieving state-of-the-art or highly competitive performance across a diverse set of challenging PDE benchmarks.
\end{itemize}

The remainder of this paper is organized as follows: Section \ref{sec: Lino} presents the proposed LiNO framework, including the problem formulation, learnable lifting decomposition, multilevel lifting strategy, and operator evolution in the lifted space. Section \ref{sec: Experiments} describes the benchmark PDE problems, datasets and experimental settings. Section \ref{sec: Results} presents quantitative and qualitative comparisons against state-of-the-art neural operator frameworks. Finally, Section \ref{sec: conclude} summarizes the key findings, discusses the limitations of the current framework, and outlines directions for future research.

%%%%%%%%%%%%%%%%%%%%%%%%%%%%%%%%%%%%%%%%%%%%%%%%%%%%%%%%%%%%%%%%%%%%%%%%%%%%%%%%%%%%%%%%%%%%%%%%%%%%%%%%%%

\section{Lifting based multiresolution neural operator} \label{sec: Lino}

\subsection{Problem Formulation}

Let \(D \subset \mathbb{R}^d\) denote a bounded spatial domain and let
\begin{equation}
\mathcal{A} = \mathcal{A}(D;\mathbb{R}^{d_a}),
\qquad
\mathcal{U} = \mathcal{U}(D;\mathbb{R}^{d_u})
\end{equation}
be separable Banach spaces representing the input and output function spaces, respectively.
We consider the problem of learning a nonlinear solution operator
\begin{equation}
\mathcal{G}^{\dagger}: \mathcal{A} \rightarrow \mathcal{U},
\qquad
u = \mathcal{G}^{\dagger}(a),
\end{equation}
where \(a(x)\) denotes the input field, which may represent an initial condition, a forcing
term, or any other functional parameter of physical relevance and \(u(x)\) denotes the corresponding solution field of an underlying parametric PDE.

Given a dataset of paired observations
\begin{equation}
\{(a_j,u_j)\}_{j=1}^{N},
\qquad
u_j = \mathcal{G}^{\dagger}(a_j),
\end{equation}
our objective is to construct a parameterized neural operator
\(\mathcal{G}_{\theta}: \mathcal{A} \rightarrow \mathcal{U},\)
that approximates the unknown mapping \(\mathcal{G}^{\dagger}\). The parameters \(\theta\) are obtained by minimizing the mean relative
\(\ell_{2}\)-error loss:
\begin{equation}
    \text{Loss}(\theta)
    \;=\;
    \frac{1}{N}\sum_{j=1}^{N}
    \frac{\bigl\|\mathcal{G}_{\theta}(a_j) - u_j\bigr\|_{2}}
         {\bigl\|u_j\bigr\|_{2}}.
    \label{eq:loss}
\end{equation}

\subsection{Lifting scheme}

The lifting scheme \cite{Sweldens1998}, originally introduced as a spatial construction of
second-generation wavelets, provides a general mechanism for building biorthogonal multiresolution decomposition. This scheme constructs wavelets directly in the spatial domain without relying on Fourier transforms, enabling adaptive representations on irregular grids, complex domains, and nonuniform sampling.

Formally, let \(L^2(D)\) denote the Hilbert space of square-integrable functions defined over the domain \(D\). A second generation multiresolution analysis of $\mathbb{R}^d$ produces a sequence of subspaces \(\{V_j\}_{j \in \mathbb{Z}} \subset L^2(D)\), where each space is spanned by scaling functions \(\varphi_{j,k}\) and wavelet functions \(\psi_{j,m}\) at resolution level \(j\). Within this framework, the lifting scheme \cite{Sweldens1998} constructs a new biorthogonal filter bank from an existing one through two elementary operations:

\begin{enumerate}
  \item {Predict step:}
    Given the coarse (scaling-function) coefficients \(\{\lambda_{j,k}\}\),
    predict the detail (wavelet) coefficients \(\{\gamma_{j,k}\}\) at the same level:
    \begin{equation}
      \label{eq:predict_classical}
      \gamma_{j,m}^{\mathrm{new}}
        = \gamma_{j,m}^{\mathrm{old}}
          - \sum_{k \in \mathcal{K}(j,m)} s_{j,k,m}\,\lambda_{j,k}.
    \end{equation}

  \item {Update step:}
    Correct the scaling coefficients using the newly computed details to
    preserve global averages (vanishing moments):
    \begin{equation}
      \label{eq:update_classical}
      \lambda_{j,k}^{\mathrm{new}}
        = \lambda_{j,k}^{\mathrm{old}}
          + \sum_{m \in \mathcal{M}(j,k)} {\tilde{s}}_{j,k,m}\,\gamma_{j,m}^{\mathrm{new}}.
    \end{equation}
\end{enumerate}
The lifting theorem \cite[Thm.~8.1]{Sweldens1998} guarantees that the resulting filter bank is always biorthogonal, regardless of the choice of the predict and update operators, and that the transform is trivially invertible by reversing and negating each step.

\begin{figure}[t]
    \centering
    \includegraphics[width=0.9\linewidth]{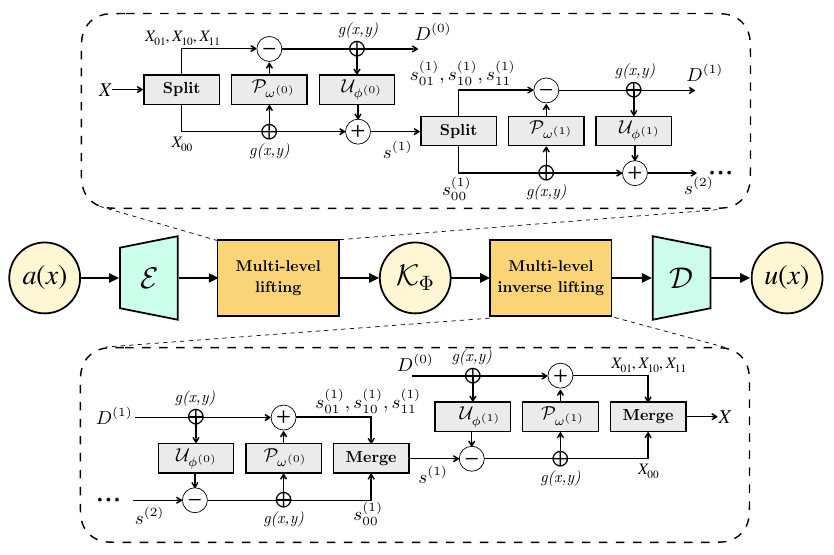}
    \caption{The input field \(a(x)\) is encoded into a latent representation using \(\mathcal{E}\). The learnable predict (\(\mathcal{P}_\omega\)) and update (\(\mathcal{U}_\phi\)) operators, augmented with spatial coordinates \(g(x,y)\), perform adaptive multiscale decomposition and reconstruction. Operator learning is carried out in the lifted space using the kernel operator \(\mathcal{K}_{\Phi}\), followed by a multi-level inverse lifting transform and a decoder \(\mathcal{D}\) to obtain the solution field \(u({x})\). The upper and lower insets illustrate the forward and inverse lifting procedures, respectively.
}
    \label{fig:architech}
\end{figure}

\subsection{Lifting neural operator}

The proposed LiNO framework first transforms the input field into a multiresolution lifting representation. The operator learning is then performed in the lifted latent space, where the solution is decomposed into coarse-scale approximation coefficients and fine-scale detail coefficients. This decomposition enables scale-aware operator learning while preserving invertibility and local consistency. In the present work, we reformulate the lifting scheme in a neural operator setting by parameterizing the predict \(( \mathcal{P}_{\omega} )\) and the update \((\mathcal{U}_\phi)\) operators with multi-layer networks, enabling adaptive multiscale decompositions learned directly from data. 

LiNO comprises the following components (see Figure~\ref{fig:architech}): feature encoding via encoder \( (\mathcal{E})\), lifting transform unit \( (\mathcal{L}) \) and corresponding inverse \( (\mathcal{L}^{-1}) \), the neural operator in lifted space \( (\mathcal{K}_\Phi) \) and a decoder \( (\mathcal{D}) \) for projection to the original physical space. The overall LiNO framework can be expressed as
\begin{equation}
\mathcal{G}_{\theta} = \mathcal{D} \circ \mathcal{L}^{-1} \circ \mathcal{K}_\Phi \circ \mathcal{L} \circ \mathcal{E}.
\end{equation}

Given an input field \( a(x) \in \mathbb{R}^{H \times W} \), we first project it into a latent feature space of dimension \(C\) using a local convolution network to get a latent tensor \( \mathcal{E}(a(x)) \rightarrow X \in \mathbb{R}^{C \times H \times W} \). This encoder maps the input field into a higher-dimensional latent representation, enabling richer cross-channel multiscale interactions. Then comes the central component of the LiNO, the lifting transform module, \( \mathcal{L}(X) \), which we discuss in detail in subsequent subsections. The operator evolution via \(\mathcal{K}_\Phi\) is performed in the lifted multiresolution domain, where the operator \(\mathcal{K}_\Phi\) is composed of convolution layers with GELU activation and a skip connection. After operator evolution, the inverse lifting transform reconstructs the latent representation, which is then projected back to physical solution space via a convolution decoder, \(\mathcal{D}\). For a 2-dimensional problem, we summarize the LiNO framework in Algorithm \ref{alg:lino}.

\subsubsection{Learnable lifting decomposition}

A central innovation of the proposed framework is the learnable lifting decomposition. Unlike classical second-generation wavelet transforms with fixed analytical filters, the proposed method learns the predict and the update operators directly from data, allowing the multiresolution basis to adapt to the intrinsic structure of the solution manifold.\\

\noindent\textbf{Anti-aliasing pre-filter:} Prior to dyadic decomposition, the latent tensor \(X\) is passed through a lightweight local smoothing operation, applied using a fixed low-pass convolution kernel. This anti-aliasing pre-filter suppresses high-frequency aliasing artifacts that would otherwise be introduced by the spatial sub-sampling.\\

\noindent\textbf{Interleaved sub-grid split:} Given the latent representation
\(
X \in \mathbb{R}^{C\times H\times W},
\)
we define the dyadic sub-lattices
\begin{equation}
\Lambda_{ij} = \left\{(2m+i,\,2n+j):\;m,n \in \mathbb{Z}\right\},\qquad(i,j)\in\{0,1\}^2.
\end{equation}

% Given the latent tensor \( X \in \mathbb{R}^{C \times H \times W},\) we partition it into four interleaved dyadic sub-grids:
% \begin{equation}
% \begin{split}
%  X_{00} &= X[:, :, 0::2, 0::2] \qquad (\text{even-even}), \\
% X_{01} &= X[:, :, 0::2, 1::2] \qquad (\text{even-odd}), \\
% X_{10} &= X[:, :, 1::2, 0::2] \qquad (\text{odd-even}), \\
% X_{11} &= X[:, :, 1::2, 1::2]  \qquad (\text{odd-odd}).
% \end{split}
% \end{equation}

The latent tensor is subsequently decomposed through restriction onto the four interleaved dyadic grids:
\begin{equation}
X_{ij} = \mathcal{R}_{ij}(X), \qquad(i,j)\in\{0,1\}^2,
\label{eq:dyadic_restriction}
\end{equation}
where
\(\mathcal{R}_{ij}\) denotes restriction to the sub-lattice \(\Lambda_{ij}.\) The component \(X_{00}\)
defines the coarse approximation component, while \(X_{01}, X_{10}, X_{11}\) encode directional fine-scale representations. This construction generalises directly to \(D\)-dimensions problem, yielding \(2^D\) sub-grids.\\

\begin{algorithm}[h!]
\caption{Lifting neural operator (LiNO)}
\label{alg:lino}
\begin{algorithmic}[1]

\Require Input field $a(x)$
\Ensure Solution field $u(x)$

\State $X \leftarrow \mathcal{E}(a)$
\Comment{encode to latent space}

\Statex \textit{// Forward lifting}

\For{$l = 0,\ldots,L-1$}

    \State
        \(
        \begin{aligned}
        X_{00} &\leftarrow X[:,:,0::2,0::2], &
        X_{01} &\leftarrow X[:,:,0::2,1::2],\\
        X_{10} &\leftarrow X[:,:,1::2,0::2], &
        X_{11} &\leftarrow X[:,:,1::2,1::2].
        \end{aligned}
        \)
    \Comment{dyadic sub-grids}

    \State $\bar X_{00} \leftarrow [X_{00},g]$
    \Comment{append coordinate grid}

    \State $(P_{01},P_{10},P_{11})
           \leftarrow
           \mathcal{P}_{\omega^{(l)}}(\bar X_{00})$
    \Comment{predict}

    \State $d_{ij}^{(l)}
           \leftarrow
           X_{ij}-P_{ij},
           \quad (i,j)\neq(0,0)$
    \Comment{detail residuals}

    \State $\bar d^{(l)}
    \leftarrow
    [d_{01}^{(l)},
     d_{10}^{(l)},
     d_{11}^{(l)},g]$
     \Comment{append coordinate grid}

    \State $s^{(l+1)}
           \leftarrow
           X_{00}
           +
           \mathcal{U}_{\phi^{(l)}}(\bar d^{(l)})$
    \Comment{update}

    \State Store
    $D^{(l)}
    =
    \{d_{01}^{(l)},
      d_{10}^{(l)},
      d_{11}^{(l)}\}$

    \State $X \leftarrow s^{(l+1)}$

\EndFor

\Statex \textit{// Operator evolution $\mathcal{K}_\Phi$}

\State $s \leftarrow \mathcal{K}_{\Phi,\mathrm{c}}(s^{(L)})$

\For{$l=0,\ldots,L-1$}
    \State $d_{01}^{(l)}
    \leftarrow
    \mathcal{K}_{\Phi,\mathrm{h}}(d_{01}^{(l)})$, ~
    $d_{10}^{(l)}
    \leftarrow
    \mathcal{K}_{\Phi,\mathrm{v}}(d_{10}^{(l)})$, ~
    $d_{11}^{(l)}
    \leftarrow
    \mathcal{K}_{\Phi,\mathrm{d}}(d_{11}^{(l)})$
\EndFor

\Statex \textit{// Inverse lifting}

\For{$l=L-1,\ldots,0$}

    \State $\bar d^{(l)}
    \leftarrow
    [d_{01}^{(l)},
     d_{10}^{(l)},
     d_{11}^{(l)},g]$

    \State $X_{00}
    \leftarrow
    s -
    \mathcal{U}_{\phi^{(l)}}(\bar d^{(l)})$
    \Comment{undo update}

    \State $\bar X_{00}
    \leftarrow
    [X_{00},g]$

    \State $(P_{01},P_{10},P_{11})
    \leftarrow
    \mathcal{P}_{\omega^{(l)}}(\bar X_{00})$

    \State $X_{01}
    \leftarrow
    d_{01}^{(l)} + P_{01}$,~
     $X_{10}
    \leftarrow
    d_{10}^{(l)} + P_{10}$,~
     $X_{11}
    \leftarrow
    d_{11}^{(l)} + P_{11}$
    \Comment{undo predict}

    \State $s \leftarrow
    \mathrm{Merge}(X_{00},X_{01},X_{10},X_{11})$

\EndFor

\State \Return $u=\mathcal{D}(s)$
\Comment{decode to physical space}

\end{algorithmic}
\end{algorithm}

\noindent\textbf{Prediction step:} Before prediction, the coarse component is augmented with a normalized
coordinate grid \(g \in \mathbb{R}^{2 \times H/2 \times W/2}\), yielding the coordinate-enriched representation
\(\bar{X}_{00} \in \mathbb{R}^{(C+2)\times H/2 \times W/2}\),
which supplies the predictor with explicit geometric context. Then the learnable predictor, \(\mathcal{P}_{\omega}\), predicts the high-frequency components from the coarse component:
\begin{equation}
(P_{01}, P_{10}, P_{11}) = \mathcal{P}_{\omega}(\bar{X}_{00}).
\end{equation}
The predictor network \(\mathcal{P}_{\omega}\) consists of convolutional layers with GELU activations and a \(1\times1\) projection bypass. The detail coefficients are then computed as prediction residuals:
\begin{equation}
    d_{ij} = X_{ij} - {P}_{ij}, ~(i,j) \neq (0,0)
\end{equation}
These residuals encode fine-scale structures that cannot be reconstructed solely from the coarse approximation.\\

\noindent\textbf{Update step:} The coarse approximation is subsequently refined using the detail coefficients. Each detail tensor is augmented with its corresponding coordinate grid, and the three resulting representations are stacked along the channel axis to form the coordinate-enriched detail set \(\bar{d} \in \mathbb{R}^{3(C+2)\times H/2 \times W/2}\), which jointly encodes multi-directional fine-scale structure alongside spatial position.  The learnable update operator \(\mathcal{U}_\phi\) (same architecture as \(\mathcal{P}_{\omega}\)) produces a correction:
\begin{equation}
    U = \mathcal{U}_\phi\!\left(\bar{d}\right),
    \label{eq:update_net}
\end{equation}
and the updated coarse representation is
\begin{equation}
    s = X_{00} + U.
    \label{eq:coarse_update}
\end{equation}
A single forward lifting transform can therefore be summarized as
\begin{equation}
\mathcal{L}^{0}(X)
=
\left(s, d_{01}, d_{10}, d_{11}\right).
\end{equation}
By construction, the predict–update operations are exactly invertible. Given \((s, d_{01}, d_{10}, d_{11})\), the inverse block first undoes the update step as
\begin{equation}
    X_{00} = s - \mathcal{U}_{\phi}(\bar d),
    \label{eq:inv_update}
\end{equation}
then recovers the remaining sub-grid components by adding back the predictions,
\begin{equation}
\begin{split}
    (P_{01}, P_{10}, P_{11}) &= \mathcal{P}_{\omega}(\bar X_{00})\\ 
    X_{ij} &= d_{ij} + P_{ij}, \qquad (i,j) \neq (0,0).
    \label{eq:inv_predict}
\end{split}
\end{equation}
~

\begin{proposition}[Exact Invertibility of \(\mathcal{L}\)]
  \label{prop:invertibility}
  Let \(\mathcal{P}_{\omega}\) and \(\mathcal{U}_\phi\) be arbitrary (possibly nonlinear)
  measurable maps.
  Then the forward lifting transform is a bijection, and its exact inverse exists.
\end{proposition}

\begin{proof}
We construct an explicit two-sided inverse. Given any \((s,\,d_{01},\,d_{10},\,d_{11})\) in the image of \(\mathcal{L}\), define \(X_{00}\) by \eqref{eq:inv_update} and \(X_{01}, X_{10}, X_{11})\) by \eqref{eq:inv_predict}. Substituting the definition \(s = X_{00} + \mathcal{U}_{\phi}(\cdot)\) into \eqref{eq:inv_update} yields \(X_{00} = X_{00} + \mathcal{U}_{\phi}(\cdot) - \mathcal{U}_{\phi}(\cdot) = X_{00}\), confirming consistency. Substituting \(d_{ij} = X_{ij} - \mathcal{P}_{\omega}(\cdot)\) into \eqref{eq:inv_predict} gives \(X_{ij} = (X_{ij} - \mathcal{P}_{\omega}(\cdot)) + \mathcal{P}_{\omega}(\cdot) = X_{ij}\). Hence \(\mathcal{L}^{-1} \circ \mathcal{L} = \mathrm{Id}\) and \(\mathcal{L} \circ \mathcal{L}^{-1} = \mathrm{Id}\), so \(\mathcal{L}\) is a bijection. Since \(\mathcal{P}_{\omega}\) and \(\mathcal{U}_{\phi}\) are arbitrary, no smoothness or linearity is required.
\end{proof}

% \begin{remark}[Connection to FNO and WNO]
% Setting the lifting wavelet basis to fixed complex exponentials $\psi_{\alpha,\beta}(x) = e^{-2\pi i \langle x, \omega \rangle}$
% reduces the scheme to the Fourier integral operator of \cite{Li2021FNO}.
% Setting it to any orthogonal mother wavelet recovers the WNO of \cite{Tripura2023WNO}.
% LiNO generalizes both by using second-generation wavelets constructed via the learnable lifting
% scheme, which do not require translation invariance, enabling application to irregular grids,
% complex domains, and non-periodic boundary conditions without modification.
% \end{remark}

\subsubsection{Multi-level lifting}

A single lifting block halves the spatial resolution of the coarse field from \(H\times W\) to \((H/2)\times(W/2)\). We stack $L$ such blocks to form the multi-level lifting module, applying the forward pass recursively to the coarse output:
\begin{equation}
  \label{eq:multilevel_fwd}
  ({s}^{(l+1)},\; {D}^{(l)})
    = \mathcal{L}^{(l)}({s}^{(l)}),
  \qquad
  {s}^{(0)} = X,
  \quad l = 0, \dots, L-1,
\end{equation}
where ${D}^{(l)} = ({d}_{01}^{(l)},{d}_{10}^{(l)},{d}_{11}^{(l)})$ collects the three detail tensors at level $l$, having resolution \(H/2^{l+1} \times W/2^{l+1}\). After $L$ levels, the spatial resolution of the coarse field is \(H/2^L \times W/2^L\). Each block has its own independent parameter set \(\{\omega^{(l)}, \phi^{(l)}\}\), allowing the predict and update operators to adapt to the varying scales. The full lifting transform is the map \(\mathcal{L} : X \mapsto \bigl({s}^{(L)},\,{D}^{(0)},\ldots,{D}^{(L-1)}\bigr)\).

\begin{corollary}
  The $L$-level lifting transform is exactly invertible for any parameter setting, with inverse given by the sequential application of the inverse transform in reverse order.
\end{corollary}
\begin{proof}
  Immediate from Proposition~\ref{prop:invertibility} applied
  inductively across all $L$ levels.
\end{proof}

\begin{remark}
  Proposition~\ref{prop:invertibility} guarantees that no information is lost in the forward and inverse lifting transform, making the decomposition and reconstruction lossless. In contrast, FNO truncates high-frequency Fourier modes, incurring an irreversible approximation error controlled by number of modes.
\end{remark}

\subsubsection{Operator evolution in lifted multiresolution space}

Following multiresolution decomposition, operator evolution is performed separately on the coarse approximation and directional detail coefficients through a hierarchy of localized convolutional operators. The coarse representation \(s^{(L)}\) is evolved using a dedicated coarse operator \(\mathcal{K}_{\Phi,\mathrm{c}}\), while the directional detail coefficients are propagated using orientation-specific operators
\(\mathcal{K}_{\Phi,\mathrm{h}}\),
\(\mathcal{K}_{\Phi,\mathrm{v}}\), and
\(\mathcal{K}_{\Phi,\mathrm{d}}\),
corresponding to horizontal, vertical, and diagonal interactions, respectively. These operators are shared across the decomposition levels and enforce scale-consistent multi-resolution dynamics while preserving the directional information. The design is inspired by the observation that similar localized physical interactions govern structures at multiple spatial scales. Each operator \(\mathcal{K}_{\Phi}\) is parameterized as a residual local convolutional block of the form
\begin{equation}
\begin{split}
Y'
=
W_3 *
\sigma\Big(
W_2 *
\sigma(
W_1 * Y
)
\Big)
+
W_{\mathrm{skip}}Y,
\end{split}
\end{equation}
where \(Y\) denotes either a coarse or directional detail coefficient tensor, \(W_1,W_2,W_3\) are learnable convolution kernels, \(*\) denotes convolution, \(\sigma(\cdot)\) represents the GELU activation function, and \(W_{\mathrm{skip}}\) is a learnable \(1\times1\) convolutional projection used to stabilize residual feature propagation.

\subsection{Computational complexity}

Consider an input feature tensor of spatial resolution $H \times W$ with latent channel dimension $C$ and $L$ lifting levels. At lifting level $l$, the feature tensor is decomposed into dyadic subbands of spatial size
\(
H/2^l \times W/2^l,
\)
where $l \in \{0,1,\dots,L-1\}$.
The predict, update, and coefficient evolution operators are implemented using localized convolutional mappings with compact kernels. Consequently, the computational cost at level $l$ scales as
\[
\mathcal{O}\left(
C^2 \frac{HW}{4^l}
\right).
\]

Summing across all decomposition levels yields
\[
\sum_{l=0}^{L-1}
\mathcal{O}\left(
C^2 \frac{HW}{4^l}
\right)
=
\mathcal{O}(C^2HW), \qquad (\text{bounded geometric series})
\]
demonstrating that the overall complexity of LiNO scales linearly with the spatial degrees of freedom.

\section{Experiments}\label{sec: Experiments}

In this section, we evaluate the proposed LiNO framework across a diverse set of benchmark problems spanning elliptic, parabolic, reaction--diffusion, and incompressible fluid dynamics systems. The considered examples include the Darcy flow equation, Poisson equation, Allen--Cahn equation, Navier--Stokes equations in both viscous and inviscid regimes, and the Gray--Scott reaction--diffusion system. These problems were selected to assess the capability of LiNO in learning operators associated with multiscale dynamics, sharp interfaces, nonlinear transport, and chaotic spatio--temporal evolution.

All experiments are performed on a GPU (NVIDIA RTX 4500 Blackwell, 32 GB memory) with CUDA 12.3. The models are trained using the Adam optimizer~\cite{adamkingma} with a cosine annealing learning rate scheduler that gradually decays the learning rate from \(10^{-3}\) to \(10^{-5}\). The hyperparameters corresponding to each benchmark problem are given in the Appendix~\ref{App_A}. We consider datasets from widely used scientific machine learning repositories and operator learning studies. In the following subsections, each benchmark problem, the governing equations, the construction of the dataset, and the experimental setup are described in detail.

% All experiments are conducted on a NVIDIA RTX 4500 Blackwell GPU with 32 GB memory using CUDA 12.3. The models are trained using the Adam optimizer~\cite{adamkingma} together with a cosine annealing learning rate scheduler, where the learning rate is gradually decayed from \(10^{-3}\) to \(10^{-5}\). The detailed hyperparameter configurations corresponding to each benchmark problem are provided in the Appendix~\ref{App_A}. We consider datasets from widely used scientific machine learning repositories and operator learning studies. The following subsections describe each benchmark problem, the corresponding governing equations, dataset construction, and experimental setup in detail.

\subsection{Allen--Cahn equation}

The Allen--Cahn equation is a nonlinear reaction--diffusion equation originally introduced to describe phase separation phenomena and anti-phase domain coarsening in multi-phase materials. Owing to its sharp moving interfaces, metastable dynamics, and strong nonlinearity, it has become a canonical benchmark for evaluating numerical methods and neural operators in learning multiscale spatiotemporal dynamics. The governing equation is given by
\begin{equation}
    \frac{\partial u(x,y,t)}{\partial t} = \phi \nabla^2 u(x,y,t) - (u(x,y,t)^3-u(x,y,t)),
\end{equation}
defined on the spatial domain \((x,y)\in(0,3)\times (0,3)\) and temporal domain \(t\in(0,20]\) with periodic boundary condition, where \(u(x,y,t)\) denotes the phase-field variable and \(\phi\) controls the interface thickness.

The dataset used for this example is adopted from \cite{WNO}, which considers the parameter setting \(\phi = 0.001\) and input initial conditions are sampled from a Gaussian random field with a covariance kernel
\begin{equation*}
K(x,y)=\tau^{\beta-1}\left(\pi^2(x^2+y^2)+\tau^2\right)^{-\beta},
\end{equation*}
where \(\tau = 15\) and \(\beta = 1\). For each sampled initial condition, the neural operator is trained to predict the phase-field solution at the final time \(t=20\). The dataset consists of \(1200\) training samples and \(300\) testing samples, each discretized on a uniform spatial resolution of \(128 \times 128\).

\subsection{Darcy flow}

The Darcy flow equation describes the steady-state flow of an incompressible fluid through a porous medium and serves as a fundamental benchmark problem for evaluating neural operators on elliptic PDEs with heterogeneous coefficients. The governing equation is given by
\begin{equation}
    -\nabla \cdot \left(\kappa(x,y)\nabla u(x,y)\right) = f(x,y), \qquad (x,y)\in(0,1)\times (0,1),
\end{equation}
where \(u(x,y)\) denotes the pressure field, \(\kappa(x,y)\) represents the permeability field of the porous medium, and \(f(x,y)\) is the forcing term.  

We consider the benchmark dataset provided in \cite{FNO}, where the forcing function is fixed as \(f(x,y)=1,\) and homogeneous Dirichlet boundary conditions are imposed. The input permeability field \(\kappa(x,y)\) is sampled from a Gaussian random field, as described in \cite{FNO}. The objective of the neural operator is to learn the nonlinear operator mapping \(\kappa(x,y)\mapsto u(x,y),\)
thereby predicting the pressure solution corresponding to an arbitrary permeability realization.

For this benchmark, we use \(800\) training samples and \(200\) testing samples, each discretized on a uniform spatial grid of resolution \(128 \times 128\).

\subsection{Poisson Equation}

The Poisson equation is a canonical elliptic PDE that arises in numerous physical applications including electrostatics, incompressible fluid dynamics, gravitational potential theory, and heat conduction. The two-dimensional Poisson equation is governed by
\begin{equation}
    -\Delta \phi(x,y) = s(x,y), \qquad (x,y)\in(0,1)\times(0,1),
\end{equation}
where $\phi(x,y)$ is the scalar potential field and $s(x,y)$ is the source term. The domain boundary is subject to Dirichlet boundary conditions. The data set is generated by analytical solutions in terms of superpositions of isotropic Gaussian functions:
\begin{equation*}
    \phi(x,y)
    =
    \sum_{k=1}^{20}
    c_k
    \exp\Big(
    -\lambda_k \big((x-x_k)^2+(y-y_k)^2\big)
    \Big),
\end{equation*}
where \(c_k\) denotes the amplitude, \((x_k,y_k)\) specifies the Gaussian center, and \(\lambda_k\) determines the spatial concentration of each Gaussian component. The Gaussian parameters are sampled uniformly according to
\[
c_k \sim \mathtt{U}(-1,1), \qquad
(x_k,y_k)\sim \mathtt{U}(0.1,0.9)^2, \qquad
\lambda_k \sim \mathtt{U}(60,80).
\]

The corresponding forcing field $s(x,y)$ and boundary values are derived analytically from the constructed solution field. For sufficiently large \(\lambda_k\), the computed solutions exhibit localized structures. This makes the problem particularly challenging for global operator learning approaches.

The neural operator is trained to approximate the potential field from the source distribution. In our example, all samples are represented on a uniform spatial grid of resolution \(128\times128\), and the dataset consists of \(800\) independently generated training samples and \(200\) testing samples with randomly sampled Gaussian parameters.

\subsection{Compressible Navier--Stokes equation}

The Navier--Stokes equation describes compressible fluid flow governed by the conservation of mass, momentum, and energy. The governing equations are expressed as
\begin{equation}
\begin{split}
\frac{\partial \rho}{\partial t} + \nabla \cdot (\rho \mathbf{u}) & \;=\, 0, \\[2mm]
\rho \left(
\frac{\partial \mathbf{u}}{\partial t}
+
\mathbf{u}\cdot\nabla \mathbf{u}
\right)
& \;=\,
-\nabla p
+
\mu \Delta \mathbf{u}
+
\left(
\xi + \frac{\mu}{3}
\right)
\nabla(\nabla\cdot\mathbf{u}),
\end{split}
\end{equation}
along with the total energy conservation equation, where \(\rho\) denotes the fluid density, \(\mathbf{u}\) is the velocity field, \(p\) represents pressure, and \(\mu\) and \(\xi\) correspond to the shear and bulk viscosity coefficients, respectively.

We consider two challenging regimes from PDEBench \cite{PDEBench, DARUS-2987_2022}: a viscous flow configuration with Mach number \(M=0.1\) and viscosity coefficients \(\mu=\xi=0.01\), and an inviscid regime with \(M=0.1\) and \(\mu=\xi=10^{-8}\). The inviscid configuration exhibits substantially sharper transport dynamics, making long-horizon prediction considerably more challenging. Both datasets employ periodic boundary conditions with random-field initial conditions generated in Fourier space.

The original PDEBench dataset provides simulations on a spatial domain \([0,1]\times [0,1]\) with spatial resolution \(512 \times 512\) and \(21\) temporal snapshots. In our experiments, the data is uniformly downsampled to a resolution of \(128 \times 128\) for computational efficiency while preserving the dominant flow characteristics. We adopt an autoregressive forecasting framework in which the first \(8\) temporal snapshots are provided as input, while the subsequent \(13\) future states are predicted recursively.

For the viscous configuration \((M=0.1,\ \mu=\xi=0.01)\), we use \(2000\) training samples and \(500\) testing samples. For the inviscid configuration \((M=0.1,\ \mu=\xi=10^{-8})\), we use \(800\) training samples and \(200\) testing samples. The learning objective is to forecast the temporal evolution of the density, pressure, and velocity fields.

\subsection{Gray--Scott reaction--diffusion system.}
The Gray--Scott model describes the evolution of two interacting chemical species whose concentrations self-organize into complex spatiotemporal structures through nonlinear reaction and diffusion processes. The governing equations are
\begin{equation}
\begin{split}
\frac{\partial u}{\partial t} &= D_u \Delta u - uv^2 + \alpha(1-u), \\[2mm]
\frac{\partial v}{\partial t} &= D_v \Delta v + uv^2 - (\alpha+\beta)v,
\end{split}
\end{equation}

where \(u\) and \(v\) denote the concentrations of the two interacting chemical species, \(D_u\) and \(D_v\) are the corresponding diffusion coefficients, while \(\alpha\) and \(\beta\) represent the feed and decay rates, respectively.

In this work, we consider the Spirals configuration of the Gray--Scott system from the well dataset \cite{thewell}, corresponding to the parameter setting \(\alpha = 0.018\) and \(\beta = 0.051\). The diffusion coefficients are fixed as \(D_u = 2\times10^{-5}\) and \(D_v = 1\times10^{-5}\). The dataset is generated using a high-order Fourier spectral solver with an implicit--explicit exponential time-differencing fourth-order Runge--Kutta scheme on a doubly periodic two-dimensional domain \([-1,1]\times [-1,1]\).

The dataset contains \(180\) training and \(20\) testing samples, each having two scalar fields corresponding to the concentrations of two chemical species. In our experiments, the spatial resolution is uniformly downsampled to \(128 \times 128\). Following the autoregressive forecasting setup, the first 8 temporal snapshots are provided as input to the model, while the subsequent 32 time steps are predicted recursively. This benchmark is particularly challenging because the dynamics simultaneously exhibit nonlinear reaction kinetics, diffusion-driven transport, and fine-scale spiral structure formation with long-range temporal dependencies.

%%%%%%%%%%%%%%%%%%%%%%%%%%%%%%%%%%%%%%%%%%%%%%%%%%%%%%%%%%%%%%%%%%%%%%%%%%%%%%%%%%%%%%%%%%%%%%%%%%%%%%%%%%%%%%%%%%%%%%%%

\section{Results and Discussion}\label{sec: Results}

\begin{table}[b!]
\centering
\caption{Mean relative $\ell_2$-error (mean $\pm$ standard deviation) computed over the test set for each neural operator across all benchmark problems}
\label{tab:l2err}
\renewcommand{\arraystretch}{2.4}
\begin{tabular}{l c c c c cc}
\toprule
& LNO & WNO & FNO & UNO & CNO &  \makecell{LiNO \\[-2pt] (proposed)} \\
\midrule
Allen--Cahn
& \makecell{1.59e-01 \\[-2pt] {\scriptsize $\pm$ 1.51e-02}}
& \makecell{1.08e-01 \\[-2pt] {\scriptsize $\pm$ 1.07e-02}}
& \makecell{1.69e-02 \\[-2pt] {\scriptsize $\pm$ 5.01e-03}}
& \makecell{4.53e-02 \\[-2pt] {\scriptsize $\pm$ 9.49e-03}}
& \makecell{2.20e-02 \\[-2pt] {\scriptsize $\pm$ 6.93e-03}}
& \makecell{\textbf{4.15e-03} \\[-2pt] {\scriptsize $\pm$ 1.26e-03}}
\\
Darcy Flow
& \makecell{2.96e-02 \\[-2pt] {\scriptsize $\pm$ 4.99e-03}}
& \makecell{2.58e-02 \\[-2pt] {\scriptsize $\pm$ 9.17e-03}}
& \makecell{7.69e-03 \\[-2pt] {\scriptsize $\pm$ 2.29e-03}}
& \makecell{8.59e-03 \\[-2pt] {\scriptsize $\pm$ 3.47e-03}}
& \makecell{5.58e-03 \\[-2pt] {\scriptsize $\pm$ 1.59e-03}}
& \makecell{\textbf{4.55e-03} \\[-2pt] {\scriptsize $\pm$ 9.08e-04}}
\\
Poisson equation
& \makecell{1.55e-01 \\[-2pt] {\scriptsize $\pm$ 3.17e-02}}
& \makecell{2.02e-01 \\[-2pt] {\scriptsize $\pm$ 6.45e-02}}
& \makecell{5.60e-02 \\[-2pt] {\scriptsize $\pm$ 2.77e-02}}
& \makecell{4.39e-02 \\[-2pt] {\scriptsize $\pm$ 2.58e-02}}
& \makecell{\textbf{1.82e-02} \\[-2pt] {\scriptsize $\pm$ 7.73e-03}}
& \makecell{2.06e-02 \\[-2pt] {\scriptsize $\pm$ 8.12e-03}}
\\
\makecell[l]{Navier--Stokes \\[-2pt] (viscous density)}
& \makecell{- \\[-2pt] }
& \makecell{1.41e-01 \\[-2pt] {\scriptsize $\pm$ 2.55e-01}}
& \makecell{7.09e-02 \\[-2pt] {\scriptsize $\pm$ 1.25e-01}}
& \makecell{2.08e-02 \\[-2pt] {\scriptsize $\pm$ 2.55e-02}}
& \makecell{3.84e-02 \\[-2pt] {\scriptsize $\pm$ 5.72e-02}}
& \makecell{\textbf{1.50e-02} \\[-2pt] {\scriptsize $\pm$ 3.32e-02}}
\\
\makecell[l]{Navier--Stokes \\[-2pt] (inviscid density)}
& \makecell{- \\[-2pt] }
& \makecell{2.51e-01 \\[-2pt] {\scriptsize $\pm$ 4.16e-01}}
& \makecell{1.66e-01 \\[-2pt] {\scriptsize $\pm$ 2.61e-01}}
& \makecell{9.77e-02 \\[-2pt] {\scriptsize $\pm$ 8.69e-02}}
& \makecell{1.14e-01 \\[-2pt] {\scriptsize $\pm$ 1.37e-01}}
& \makecell{\textbf{6.77e-02} \\[-2pt] {\scriptsize $\pm$ 1.14e-01}}
\\
\multirow{2}{*}{Gray--Scott} \hfill (\(u\))
& \makecell{- \\[-2pt] }
& \makecell{9.86e-02 \\[-2pt] {\scriptsize $\pm$ 4.22e-02}}
& \makecell{1.63e-01 \\[-2pt] {\scriptsize $\pm$ 2.90e-02}}
& \makecell{4.48e-02 \\[-2pt] {\scriptsize $\pm$ 1.92e-02}}
& \makecell{7.79e-02 \\[-2pt] {\scriptsize $\pm$ 1.71e-02}}
& \makecell{\textbf{9.98e-03} \\[-2pt] {\scriptsize $\pm$ 6.79e-03}}
\\
\hfill (\(v\))
& \makecell{- \\[-2pt] }
& \makecell{3.93e-01 \\[-2pt] {\scriptsize $\pm$ 5.44e-02}}
& \makecell{6.94e-01 \\[-2pt] {\scriptsize $\pm$ 2.57e-01}}
& \makecell{1.99e-01 \\[-2pt] {\scriptsize $\pm$ 5.31e-02}}
& \makecell{3.18e-01 \\[-2pt] {\scriptsize $\pm$ 9.60e-02}}
& \makecell{\textbf{5.79e-02} \\[-2pt] {\scriptsize $\pm$ 4.27e-02}}
\\

\bottomrule
\end{tabular}
\end{table}

The LiNO is compared against representative neural operator baselines, including LNO, WNO, FNO, UNO, and CNO. Performance is assessed using the relative $\ell_2$-error together with computational efficiency metrics including training time, memory consumption, and learnable parameter count. To further assess each method's capability to preserve localized fine-scale structures, we present qualitative comparisons of the learned solution fields and the corresponding point-wise error maps.

Table~\ref{tab:l2err} reports the mean relative $\ell_2$-errors along with standard deviation over the test samples across all considered benchmark problems. The lowest mean relative $\ell_2$-errors is highlighted in bold. Overall, LiNO consistently achieves either the best or highly competitive performance across all problems, demonstrating strong robustness across a broad range of physical regimes and operator learning tasks. The most significant performance improvements are observed for strongly multiscale and transport-dominated systems such as the Allen--Cahn equation, inviscid Navier--Stokes, and the Gray--Scott reaction--diffusion system. These problems contain sharp interfaces, localized coherent structures, nonlinear dynamics, and long-range temporal dependencies that are particularly challenging for operator learning. In these settings, the adaptive lifting decomposition employed by LiNO appears especially effective in preserving localized fine-scale information throughout. We notice that the LNO could not be evaluated on several large-scale spatiotemporal benchmarks due to prohibitively high memory consumption and extreme training time, we indicate this with a dash in respective entries in the table. 

\begin{table}[t!]
\centering
\caption{Total training time (in minutes) for each neural operator across all benchmark problems}
\label{tab:time}
\begin{tabular}{lcccccc}
\hline
Method & LNO & WNO & FNO & UNO & CNO & LiNO \\
\hline
Allen--Cahn & 335.35 & 46.40 & 20.05 & 25.98 & 95.13 & 42.41 \\
Darcy Flow & 221.90 & 70.87 & 11.81 & 22.53 & 61.72 & 16.71 \\
Poisson equation & 221.99 & 70.84 & 11.88 & 22.51 & 61.83 & 16.65 \\
\makecell[l]{Navier--Stokes \\[-2pt] (viscous)} & - & 864.87 & 445.93 & 274.95 & 1413.01 & 765.61 \\
\makecell[l]{Navier--Stokes \\[-2pt] (inviscid)} & - & 399.01 & 178.21 & 210.15 & 566.02 & 374.93 \\
Gray--Scott & - & 215.13 & 49.98 & 58.49 & 385.19 & 274.15 \\
\hline
\end{tabular}
\end{table}

\begin{figure}[b!]
    \centering
    \includegraphics[width=0.96\linewidth]{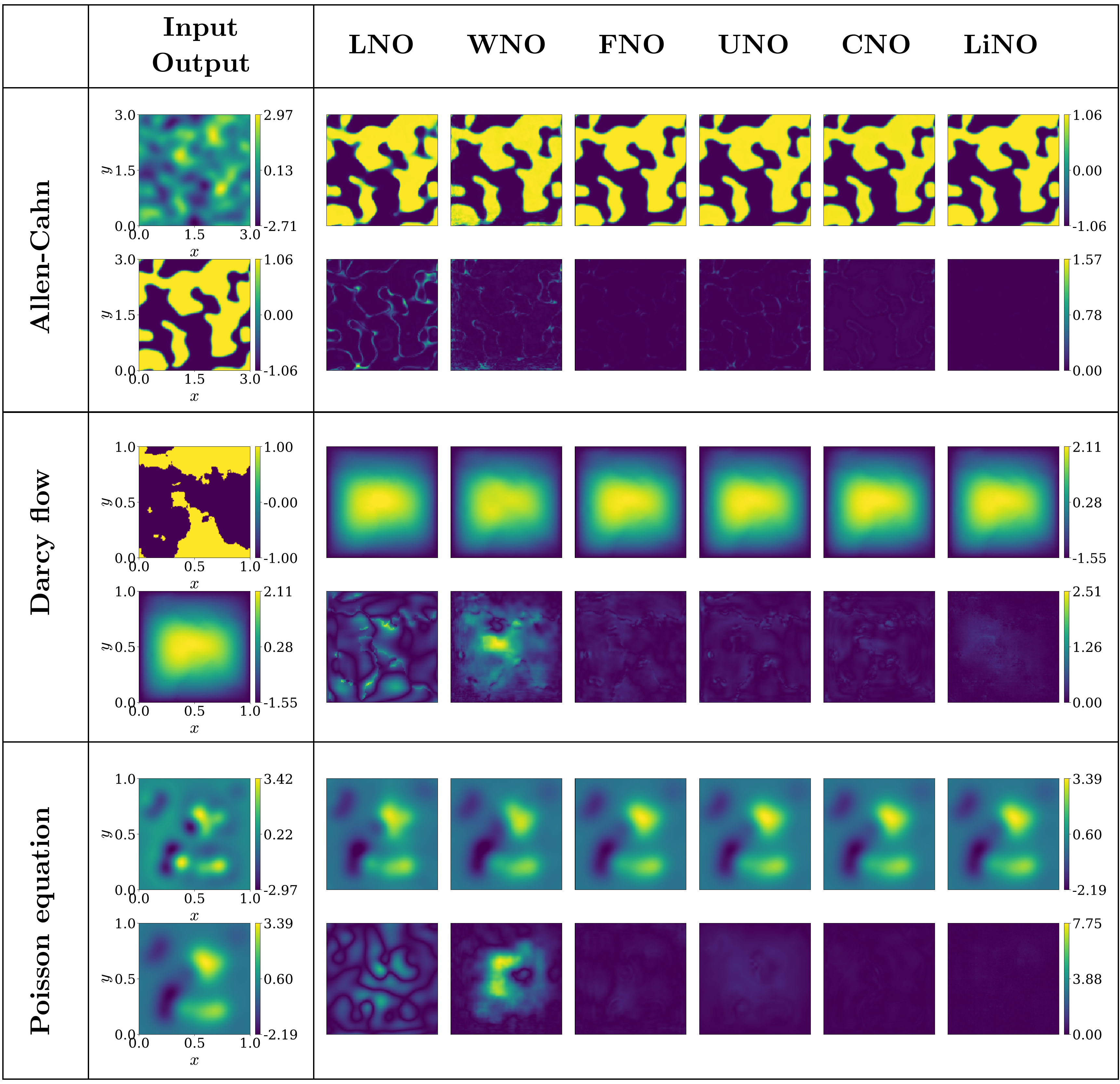}
    \caption{\textbf{Qualitative comparison of neural operator predictions on steady-state PDE benchmarks.}
    \textbf{From top to bottom row:} Allen--Cahn equation, Darcy flow, and Poisson equation.
    \textbf{Left Panel:} The input parameter field (top) and the corresponding reference solution (bottom) of a test sample.
    \textbf{Right Panel:} The predicted solution field (top) and the corresponding point-wise absolute error (bottom) for each neural operator model.}
    \label{fig:output0}
\end{figure}

\begin{figure}[t!]
    \centering
    \includegraphics[width=0.96\linewidth]{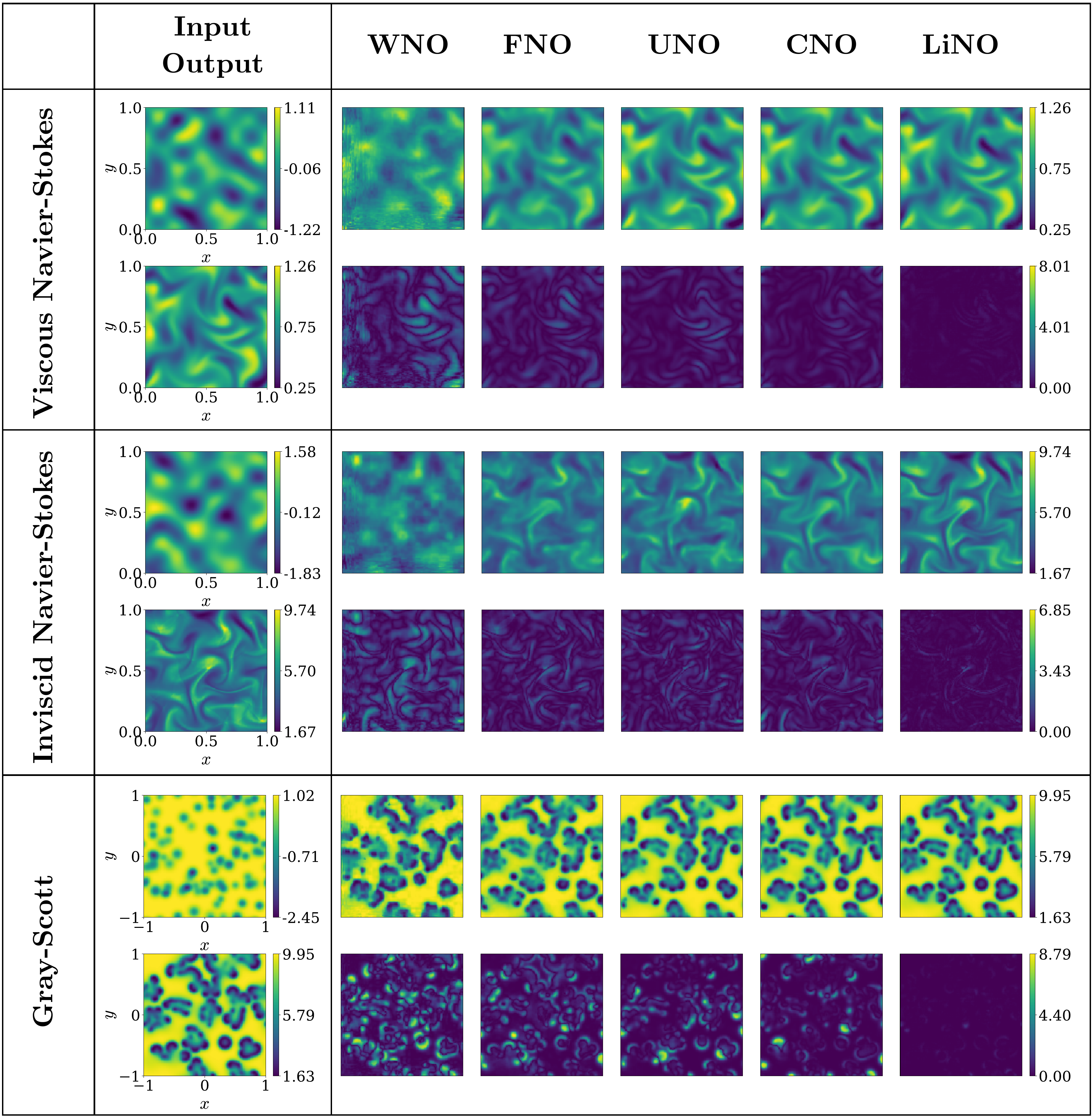}
    \caption{\textbf{Qualitative comparison of neural operator predictions on time-dependent dynamical systems.}
    \textbf{From top to bottom row:} Viscous Navier--Stokes (density plot), inviscid Navier--Stokes (density plot), and Gray--Scott reaction--diffusion equations (concentration $u$).
    \textbf{Left Panel:} The initial condition (top) and the corresponding ground-truth solution (bottom) at the final prediction time of a test sample.
    \textbf{Right Panel:} The predicted final-time solution (top) and the corresponding point-wise absolute error (bottom) for each neural operator model.}
    \label{fig:output1}
\end{figure}

Table \ref{tab:time} reports the total training time for each benchmark problem in minutes. As expected, purely spectral methods such as FNO generally exhibit the lowest computational cost due to the highly efficient FFT-based convolutions. Nevertheless, LiNO remains computationally competitive despite employing a fully learnable multilevel lifting transform. Although LiNO introduces additional computational cost associated with multilevel lifting operations, it shows substantial gains in predictive accuracy and multiscale representation capability.

% To further assess qualitative reconstruction fidelity, representative solution visualizations together with pointwise absolute error maps are provided in Appendix~A for all considered benchmark problems. These comparisons demonstrate that LiNO more accurately preserves localized interfaces, fine-scale transport structures, and nonlinear pattern dynamics while substantially reducing spatially concentrated prediction errors. In particular, LiNO exhibits improved reconstruction fidelity in regions containing sharp transitions, coherent structures, and multiscale interactions where competing approaches often introduce excessive smoothing or accumulated autoregressive errors.

In addition to predictive accuracy and training time, Appendix \ref{App_B} also reports peak GPU memory consumption in Table \ref{tab:memory} and total trainable parameter counts in Table \ref{tab:param} for all considered models. Despite its strong predictive performance, LiNO typically requires substantially fewer trainable parameters than competing multiscale architectures such as WNO and UNO, and the memory footprint remains manageable across all considered benchmarks and substantially lower than several heavy hierarchical baselines.

Furthermore, qualitative comparisons are presented in Figures \ref{fig:output0} and \ref{fig:output1}. Across both steady-state and time-dependent benchmarks, LiNO more accurately reconstructs the solution fields, as evidenced by consistently lower point-wise errors throughout the domain. In contrast, competing methods tend to produce smoother predictions and exhibit larger errors in regions containing fine-scale structures.

Collectively, these results demonstrate that LiNO provides an effective balance between predictive accuracy, multiscale representation capability, and computational efficiency. The proposed adaptive lifting formulation appears particularly advantageous for PDE systems involving localized structures, sharp transitions, and chaotic spatiotemporal dynamics, where preserving fine-scale information is critical for stable and accurate operator learning.

%%%%%%%%%%%%%%%%%%%%%%%%%%%%%%%%%%%%%%%%%%%%%%%%%%%%%%%%%%%%%%%%%%%%%%%%%%%%%%%%%%%%%%%%%%%%%%%%%%%%%%%%%%%%%%%%%%%%%%%%

\section{Conclusion and Limitations} \label{sec: conclude}

In this work, we proposed LiNO, a lifting-based neural operator that performs operator learning in an adaptive multiresolution space. The framework combines learnable lifting transforms with operator evolution on coarse and detail representations, enabling scale-aware operator learning while preserving multiscale information. The proposed method was evaluated on a diverse set of benchmark PDEs spanning elliptic, reaction--diffusion, and fluid dynamics problems. Across most benchmarks, LiNO achieved the best predictive accuracy and consistently demonstrated strong performance on problems characterized by localized structures, sharp interfaces, and long-horizon dynamics. Furthermore, LiNO remained computationally competitive while requiring substantially fewer trainable parameters than several existing multiscale neural operator architectures.

Despite these promising results, several limitations remain. The current framework relies on a recursive dyadic decomposition and is therefore naturally suited to grid resolutions that can be repeatedly subdivided by powers of two. Although padding, interpolation, or truncation can accommodate arbitrary resolutions, such preprocessing may introduce artifacts and degrade the benefits of the multiresolution representation. In addition, the present formulation is restricted to structured Cartesian grids, extending the lifting decomposition to unstructured meshes and complex geometries requires nontrivial redesign of the splitting, prediction, and update operators. Finally, a rigorous theoretical guarantees regarding approximation properties and stability under perturbations remain unavailable. Consequently, extending LiNO to irregular geometries and establishing rigorous theoretical guarantees constitute important directions for future work.

%%%%%%%%%%%%%%%%%%%%%%%%%%%%%%%%%%%%%%%%%%%%%%%%%%%%%%%%%%%%%%%%%%%%%%%%%%%%%%%%%%%%%%%%%%%
%\clearpage

\iffalse 
\subsubsection*{CRediT authorship contribution statement}
\noindent  \textbf{Himanshu Pandey:} Conceptualization, Methodology, Programming, Formal analysis, Writing - original draft.   
\textbf{Ratikanta Behera:} Conceptualization, Methodology, Formal analysis, Supervision, Writing - review \& editing.
\fi 

\subsubsection*{Acknowledgment}
The author Subham Patel would like to acknowledge the Axis Bank Centre for Mathematics and Computing, Indian Institute of Science, Bangalore, India, for the financial assistance to carry out this research work in IISc Mathematics Initiative, Department of Mathematics, Indian Institute of Science, Bangalore, India.

\noindent 

% \subsection*{Ethical Approval}
% In this manuscript, all the authors have agreed to authorship and approved the manuscript with consent for submission.
 
\subsubsection*{Funding details}
\noindent No funding source is available for this research.

\subsubsection*{Data availability statements}
All source code to reproduce the results of this study will be made available on request.
% Source code used in this study is available at {\url{https://github.com/phimanshu-22/AWPINN.git}}.

\subsubsection*{Conflicts of interest and declarations}	
The authors declare that they do not have any conflicts of interest. In addition, they also declare that this work is not under consideration anywhere. 	
	
\subsubsection*{Declaration of Generative AI Use}
During the preparation of this manuscript, the authors used AI tools for English language correction.
	
%%%%%%%%%%%%%%%%%%%%%%%%%%%%%%%%%%%%%%%%%%%%%%%%%%%%%%%%%%%%%%%%%%%%%%%%%%%%%%%%%%%%%%%%%%%%%%%%%%%%%%%%%%%%%%%%%%%%%%%%%	
		
		%%%% Acknowledgment %%%%%%%%
%		\section*{Acknowledgement}
%		The author would like to thank the referees for the helpful
%		suggestions.
		
%%%% Bibliography  %%%%%%%%%%
	
\bibliography{ReferencesSS}

%%%%%%%%%%%%%% Appendix %%%%%%%%%%%%%%

\newpage
\appendix
\section{Hyperparameters}\label{App_A}

\begin{table*}[h]
\centering
\caption{Hyperparameters used for LiNO across all benchmark PDEs. Dashes indicate parameters not applicable to a given problem.}
\label{tab:hyperparams}
\resizebox{\textwidth}{!}{%
\begin{tabular}{lcccccc}
\toprule
\textbf{Hyperparameter} & \textbf{Allen-Cahn} & \textbf{Darcy} & \textbf{Poisson} & \makecell{\textbf{NS} \\ \textbf{(viscous)}} & \makecell{\textbf{NS} \\ \textbf{(inviscid)}} & \textbf{Gray-Scott} \\
\midrule
\multicolumn{7}{l}{\textbf{Data}} \\
\midrule
Input fields        & \makecell{\texttt{Initial} \\ [-3pt] \texttt{condition}}    & \makecell{\texttt{Permeability} \\ [-3pt] \texttt{field}}      & \makecell{\texttt{Source} \\ [-3pt] \texttt{term}}       & \makecell{\texttt{Initial} \\ [-3pt] \texttt{conditions}} & \makecell{\texttt{Initial} \\ [-3pt] \texttt{conditions}} & \makecell{\texttt{Initial} \\ [-3pt] \texttt{concentrations}} \\
Target fields & \makecell{\texttt{Phase} \\ [-3pt] \texttt{field}}       & \makecell{\texttt{Pressure} \\ [-3pt] \texttt{field}}  & \makecell{\texttt{Potential} \\ [-3pt] \texttt{field}}       & \makecell{\texttt{Flow} \\ [-3pt] \texttt{fields}} & \makecell{\texttt{Flow} \\ [-3pt] \texttt{fields}} & \makecell{\texttt{Final} \\ [-3pt] \texttt{concentrations}} \\
Resolution       & $128 \times 128$                  & $128 \times 128$                & $128 \times 128$                & $128 \times 128$                  & $128 \times 128$                  & $128 \times 128$    \\
$n_\text{train}$     & 1200               & 800              & 800              & 2000               & 800                & 180  \\
$n_\text{test}$      & 300                & 200              & 200              & 500                & 200                & 20   \\
$T_\text{in}$        & --                  & --                & --                & 8                  & 8                  & 8    \\
$T_\text{out}$       & --                  & --                & --                & 13                 & 13                 & 32   \\
\midrule
\multicolumn{7}{l}{\textbf{Training}} \\
\midrule
Epochs               & 500  & 500  & 500  & 500  & 500  & 500  \\
Batch size           & 128  & 64   & 128  & 32   & 20   & 16   \\
Initial lr $\eta$ & 1e-3 & 1e-3 & 1e-3 & 1e-3 & 1e-3 & 1e-3 \\
Minimum lr    & 1e-5 & 1e-5 & 1e-5 & 1e-5 & 1e-5 & 1e-5 \\
Optimizer            & Adam & Adam & Adam & Adam & Adam & Adam \\
Scheduler & \makecell{Cosine \\ [-3pt] Annealing} & \makecell{Cosine \\ [-3pt] Annealing} & \makecell{Cosine \\ [-3pt] Annealing} & \makecell{Cosine \\ [-3pt] Annealing} & \makecell{Cosine \\ [-3pt] Annealing} & \makecell{Cosine \\ [-3pt] Annealing} \\
\midrule
\multicolumn{7}{l}{\textbf{Architecture}} \\
\midrule
Width $C$            & 64   & 64   & 32   & 32   & 64   & 32   \\
Levels $L$           & 4    & 4    & 4    & 4    & 3    & 4    \\
$\mathcal{P}_{\omega}$, $\mathcal{U}_{\phi}$ hidden      & 32   & 64   & 32   & 32   & 64   & 32   \\
$\mathcal{P}_{\omega}$, $\mathcal{U}_{\phi}$ conv layers & 3+1$\times$1 & 3+1$\times$1 & 4+1$\times$1 & 3+1$\times$1 & 3+1$\times$1 & 3+1$\times$1 \\
$\mathcal{K}_{\Phi}$ Operator  & 3 conv & 3 conv & 3 conv & 3 conv & 3 conv & 3 conv \\
\bottomrule
\end{tabular}%
}
\end{table*}

% \newpage

% \renewcommand{\thetable}{B.\arabic{table}} 
\setcounter{table}{0}   
\section{Computational efficiency}\label{App_B}

\begin{table}[h]
\centering
\caption{Maximum GPU memory usage (in GB) recorded during training of each neural operator across all benchmark problems}
\label{tab:memory}

\begin{tabular}{lcccccc}
\hline
Method & LNO & WNO & FNO & UNO & CNO & LiNO \\
\hline
Allen--Cahn & 22.37 & 7.65 & 6.87 & 8.64 & 24.72 & 11.59 \\
Darcy Flow & 22.37 & 7.61 & 7.39 & 11.36 & 24.72 & 9.65 \\
Poisson equation & 22.37 & 7.61 & 7.39 & 11.36 & 24.72 & 9.65 \\
\makecell[l]{Navier--Stokes \\[-2pt] (viscous)} & - & 11.61 & 18.52 & 12.02 & 28.36 & 22.46 \\
\makecell[l]{Navier--Stokes \\[-2pt] (inviscid)} & - & 19.20 & 18.52 & 23.27 & 28.38 & 26.78 \\
Gray--Scott & - & 13.41 & 12.20 & 13.94 & 17.14 & 26.77 \\
\hline
\end{tabular}
\end{table}

\newpage
~
\begin{table}[t]
\centering
\caption{Total count of learnable parameters for each neural operator across all benchmark problems}
\label{tab:param}
\begin{tabular}{lcccccc}
\hline
Method & LNO & WNO & FNO & UNO & CNO & LiNO \\
\hline
Allen--Cahn & 27937 & 31740545 & 8409729 & 32870273 & 10555073 & 1200961 \\
Darcy Flow & 27937 & 21259009 & 8413953 & 32870273 & 10555073 & 1533249 \\
Poisson equation & 27937 & 21259009 & 8413953 & 32870273 & 10555073 & 1533249 \\
\makecell[l]{Navier--Stokes \\[-2pt] (viscous)} & - & 7937572 & 8412420 & 8218644 & 2035156 & 543460 \\
\makecell[l]{Navier--Stokes \\[-2pt] (inviscid)} & - & 31742724 & 8412420 & 32871460 & 2035156 & 1717188 \\
Gray--Scott & - & 7936930 & 2104482 & 8218322 & 2024786 & 542882 \\
\hline
\end{tabular}
\end{table}

% \newpage
% \setcounter{figure}{0}
% \section{Qualitative comparison}\label{App_C}

\end{document}